\documentclass{article}

\usepackage{xurl}

\usepackage{microtype}
\usepackage{graphicx}
\usepackage{booktabs} %

\usepackage{hyperref}

\usepackage[accepted]{icml2023}

\usepackage{amsmath}
\usepackage{amssymb}
\usepackage{mathtools}
\usepackage{amsthm}

\usepackage[capitalize,noabbrev]{cleveref}

\theoremstyle{plain}
\newtheorem{theorem}{Theorem}[section]

\theoremstyle{definition}

\theoremstyle{remark}

\usepackage[textsize=tiny]{todonotes}
\usepackage{adjustbox}
\usepackage{multirow}
\usepackage{listings} %

\usepackage{tikz}

\newcommand{\bind}{%
  \mathrel{%
\begin{tikzpicture}[x=0.75pt,y=0.75pt,yscale=-1,xscale=1]
\draw   (0,5) .. controls (0,2.24) and (2.24,0) .. (5,0) .. controls (7.76,0) and (10,2.24) .. (10,5) .. controls (10,7.76) and (7.76,10) .. (5,10) .. controls (2.24,10) and (0,7.76) .. (0,5) -- cycle ;
\draw    (5,0) -- (1.57,8.57) ;
\draw    (8.47,1.43) -- (5,10) ;
\draw    (0.33,3.33) -- (9.67,3.33) ;
\draw    (0.33,6.67) -- (9.67,6.67) ;
\end{tikzpicture}
  }%
}

\icmltitlerunning{Recasting Self-Attention with Holographic Reduced Representations}

\begin{document}

\twocolumn[
\icmltitle{Recasting Self-Attention with Holographic Reduced Representations}

\begin{icmlauthorlist}
\icmlauthor{Mohammad Mahmudul Alam}{umbc}
\icmlauthor{Edward Raff}{umbc,lps,bah}
\icmlauthor{Stella Biderman}{lps,bah,eai}
\icmlauthor{Tim Oates}{umbc}
\icmlauthor{James Holt}{lps}
\end{icmlauthorlist}

\icmlaffiliation{umbc}{Department of Computer Science and Electrical Engineering, University of Maryland, Baltimore County, Baltimore, MD, USA}
\icmlaffiliation{lps}{Laboratory for Physical Sciences, College Park, MD, USA}
\icmlaffiliation{bah}{Booz Allen Hamilton, McLean, VA, USA}
\icmlaffiliation{eai}{EleutherAI}

\icmlcorrespondingauthor{Edward Raff}{Raff\_Edward@bah.com}
\icmlcorrespondingauthor{Tim Oates}{oates@cs.umbc.edu}

\icmlkeywords{Self-attention, Holographic Reduced Representations}

\vskip 0.3in
]

\printAffiliationsAndNotice{}%

\begin{abstract}

In recent years, self-attention has become the dominant paradigm for sequence modeling in a variety of domains. However, in domains with very long sequence lengths the $\mathcal{O}(T^2)$ memory and $\mathcal{O}(T^2 H)$ compute costs can make using transformers infeasible. Motivated by problems in malware detection, where sequence lengths of $T \geq 100,000$ are a roadblock to deep learning, we re-cast self-attention using the neuro-symbolic approach of Holographic Reduced Representations (HRR). In doing so we perform the same high-level strategy of the standard self-attention: a set of queries matching against a set of keys, and returning a weighted response of the values for each key. Implemented as a ``Hrrformer'' we obtain several benefits including $\mathcal{O}(T H \log H)$ time complexity, $\mathcal{O}(T H)$ space complexity, and convergence in $10\times$ fewer epochs. Nevertheless, the Hrrformer achieves near state-of-the-art accuracy on LRA benchmarks and we are able to learn with just a single layer. Combined, these benefits make our Hrrformer the first viable Transformer for such long malware classification sequences and up to $280\times$ faster to train on the Long Range Arena benchmark. Code is available at \url{https://github.com/NeuromorphicComputationResearchProgram/Hrrformer}

\end{abstract}

\section{Introduction}

Self-attention has risen to prominence due to the development of transformers \citep{transformer} and their recent successes in machine translation, large language modeling, and computer vision applications. The fundamental construction of self-attention includes a triplet of ``queries, keys, and values'', where the response is a weighted average over the values based on the query-key interactions. This results in a quadratic memory and computational complexity, that has inhibited the use of Transformers to those without significant GPU infrastructure and prevented applications to longer sequences. Ever since, a myriad of approaches has been proposed to approximate the self-attention mechanism, with the vast majority trading some amount of accuracy for speed or memory use. The ``market'' of self-attention strategies currently offers various trade-offs in the total package of speed, memory use, and accuracy.

\begin{figure}[!t]
\centerline{\includegraphics[width=\columnwidth]{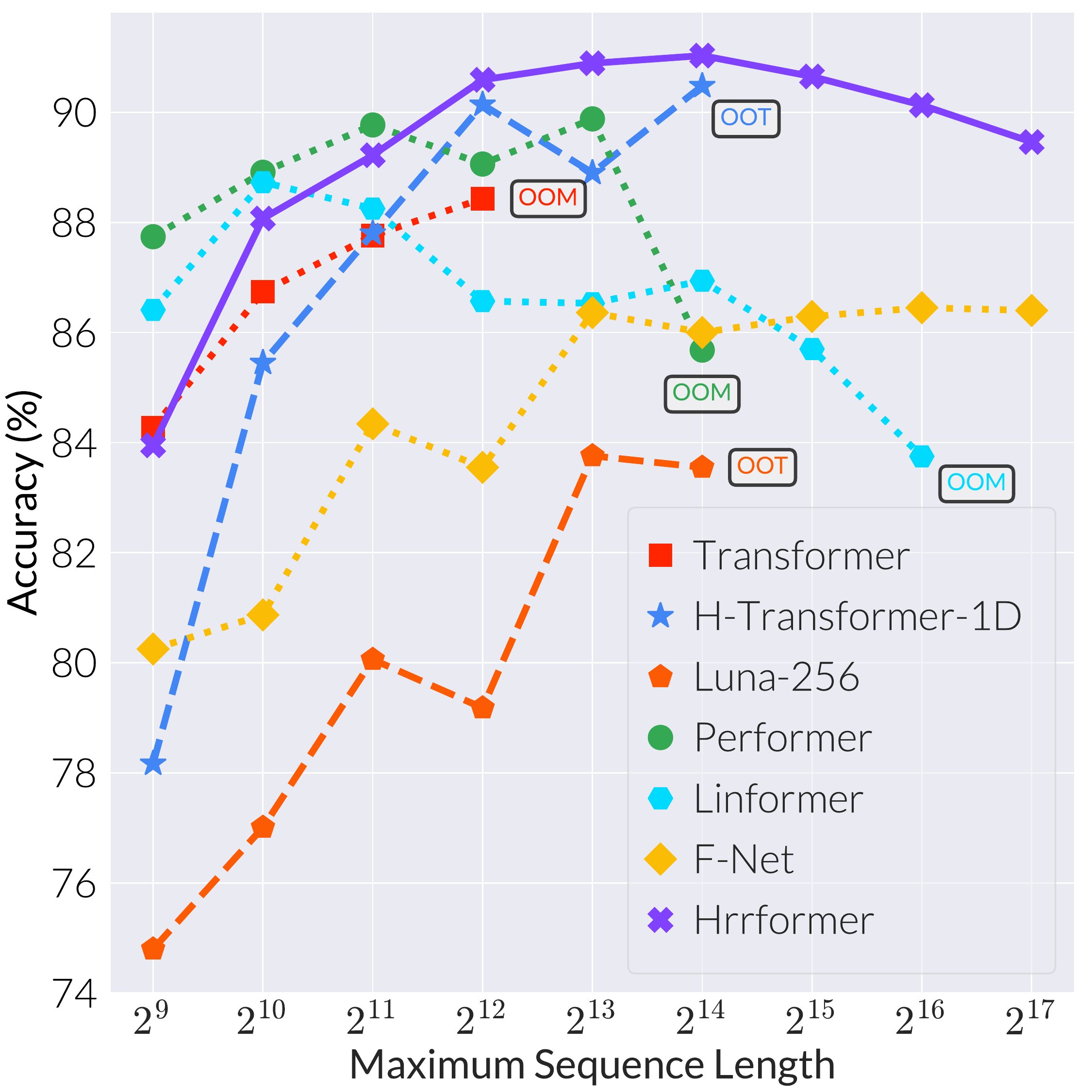}}
\caption{Our primary result, comparison of our Hrrformer with other self-attention models in EMBER malware classification dataset. Most prior methods fail early by running Out Of Memory (OOM) or Time (OOT). Hrrformer is presented in a \textit{solid} line and achieves the best accuracy, scales to longer sequences.  The two prior best models according to the Long Range Arena, H-Transformer-1D and Luna-256, are in the \textit{dashed} lines, and do not perform as well as the LRA would have indicated in speed or accuracy. The rest of the models are in the \textit{dotted} line. 
} 
\label{fig:ember_acc}
\end{figure}

We test our method in two settings: using the Long Range Arena (LRA) to compare with prior approaches and a real-world task in malware detection. These results show several benefits to the Hrrformer: it is near state-of-the-art in terms of accuracy, and one of only two methods to improve upon the original Transformer for \textit{all} tasks in the LRA. The Hrrformer sets a new benchmark for state-of-the-art speed and memory use, processing $28\times$ more samples/second and using $79.15\%$ less memory than the best prior art for each respective metric. The Hrrformer converges in $10\times$ fewer epochs and is effective with just a single layer. Combined this makes the Hrrformer up to $280\times$ times faster to train. On our malware classification task, we find that the relative accuracies of Transformer models change from the LRA benchmark, but that our Hrrformer still obtains the best accuracy and scales the best with sequence length up to $T=131,072$, as demonstrated in \autoref{fig:ember_acc}.
\par 
The remainder of our manuscript is organized as follows. Work related to our own, as well as adjacent techniques beyond our study's scope, is reviewed in \autoref{sec:related_work}. The recasting of attention in our Hrrformer is a simple procedure demonstrated in \autoref{sec:hrrformer}, which redefines the $\textit{Attention}$ function using HRR, and multi-headed self-attention then continues as normal.  We then demonstrate these benefits in \autoref{sec:experiments}, showing Hrrformer is consistently one of the best methods with respect to accuracy and considerably faster thanks to reduced memory usage, the number of layers, and epochs needed to converge. In \autoref{sec:conclusion} we draw conclusions from out work.

\section{Related Works} \label{sec:related_work}
Since the introduction of the Self-Attention mechanism and the transformer architecture, considerable research has occurred to mitigate its computational burdens. Though not explicit in much of the current literature, many of these approaches resemble strategies for improving Support Vector Machines that have similar complexity. This includes projection
~\cite{10.1145/2783258.2783364} 
to a lower dimension \cite{linformer}, finding/creating sparse structure in the correlations
~\cite{Wang2014} by~\cite{reformer,sparse,sinkhorn,longformer,bigbird},
using randomized features
~\cite{Rahimi2007,NIPS2016_6180} by ~\cite{performers}, 
factorized or budgeted representations
~\cite{si2016computatio,Wang2010a} by~\cite{Xiong2021,Ma2021}, 
and creating simplified linear approximations
~\cite{Wang2011,Kantchelian:2014:LCP:2969033.2969189} by ~\cite{linear-trans}.
Other more differentiated approaches include the hierarchical decomposition of the correlations (by ~\cite{zhu-soricut-2021-h}), and approaches that replace self-attention entirely with alternative ``mixing'' strategies \cite{synthesizer,Lee-Thorp2021}. To the best of our knowledge, ours is the first work that attempts to re-create the same logic of self-attention with the HRR. 
\par 
Among these prior methods, we note that F-Net~\cite{Lee-Thorp2021} is the most closely related as both F-Net and HRR rely upon the Fast Fourier Transform (FFT) as a fundamental building block. While F-Net does not approximate self-attention so much as replace it with an alternative ``mixing'' procedure, we include it due to its relevance in using the FFT. Our results will show significant improvement over F-Net, highlighting the value of a neuro-symbolic approach to reconstructing the same logic as opposed to using the FFT as a generic differentiable mixing strategy.
\par 
The HRR has seen successful use in cognitive science research \cite{Jones2007,Blouw2013,Stewart2014,Blouw2016,Eliasmith2012,Singh3667,10.3389/fninf.2013.00048}, but comparatively little application in modern deep learning. The symbolic properties have been previously used in knowledge graphs~\cite{10.5555/3016100.3016172} and multi-label classification~\cite{Ganesan2021}. There is limited use of HRRs for sequential modeling. \cite{10.5555/2987061.2987066} proposed an HRR-based Recurrent Neural Network (RNN), while other work has used complex numbers inspired by HRRs but not actually used the corresponding operations~\cite{Danihelka2016}. An older alternative to the HRR, the Tensor Product Representation (TPR) ~\cite{SMOLENSKY1990159} has been used to endow associative memories ~\cite{pmlr-v119-le20b} and RNNs with enhanced functionality ~\cite{huang-etal-2018-tensor,NEURIPS2018_a274315e}. Compared to these prior works, we are re-casting the logic into HRRs, rather than augmenting the logic. However, we slightly abuse the assumptions of HRRs to make our method work. A strategic design allows us to effectively remove additionally created noise via the softmax function. In addition, the TPR's complexity is exponential in the number of sequential bindings, making it a poor choice for tackling the scaling problems of self-attention.
\par 
Other recent approaches to sequential modeling such as Legendre Memory Units \cite{NIPS2019_9689}, IGLOO~\cite{Sourkov2018}, and State Space Models\cite{Gu2021,Goel2022,Gu2021a,Gu2020} are highly promising. We consider these, along with RNNs, beyond the scope of our work. Our goal is to explore the value of re-casting self-attention within the neuro-symbolic framework of HRR. As such, other sequence modeling approaches are out of scope.%
\par 
The need for both less memory and extension to very long sequences is also important in malware detection. Processing malware from raw bytes has been found to be one of the most robust feature types in the face of common malware obfuscations~\cite{Aghakhani2020}, but simple n-gram based features have been maligned for being unable to learn complex sequential information when executable can be tens of kilobytes on the small side and hundreds of megabytes on the larger side~\cite{Kephart:1995:BID:1625855.1625983, Abou-Assaleh2004, Kolter:2006:LDC:1248547.1248646, Kilograms_2019, Zak2017}. Given that a maximum $T=200M$ is realistic, many strategies to handle such sequence lengths have been developed. These include attempts to create ``images'' from malware~\cite{Nataraj:2011:MIV:2016904.2016908, Liu2016a}, using compression algorithms as a similarity metric ~\cite{Li2004, Walenstein2007, Borbely2015, S.Resende2019,Menendez2019,raff_lzjd_2017,Raff2020}, and attempts to scale 1D-convolutional networks over raw bytes~\cite{Krcal2018,MalConv,Raff2020b}. 
\par 
We will use the Ember \cite{ember} dataset for malware detection as a real-world test of our new self-attention for processing long sequences. It has been observed empirically  that ``best practices'' developed in the machine learning, computer vision, and natural language processing communities do not always transfer to this kind of data. For example, this phenomenon has been observed with CNNs~\cite{MalConv} and Transformers for malicious URL detection~\cite{Rudd2020}. Most recently, ~\cite{10.1145/3494110.3528242} attempted to apply Transformers to raw byte prediction and had to use a chunked attention that limits the attention window \cite{sukhbaatar-etal-2019-adaptive}. Using Hrrformer we show much longer sequence processing than this prior work, while simultaneously demonstrating that our method generalizes to a domain that is notorious for a lack of transfer. This increases our confidence in the effectiveness of our method.  Notably, the two current state-of-the-art Transformers as measured by the Long Range Arena (LRA)~\cite{lra} benchmarks do not pass this test, performing considerably worse on the malware task.

\section{Attention with Holographic Reduced Representations} \label{sec:hrrformer}
The HRR operation allows assigning abstract concepts to numerical vectors, and performing \textit{binding} ($\bind$) and \textit{unbinding} operations on those concepts via the vectors. One could bind ``red'' and ``cat'' to obtain a ``red cat''. The vectors can also be added, so ``red'' $\bind$ ``cat'' + ``yellow'' $\bind$ ``dog'' represents a ``red cat and yellow dog''. An inverse operator $\dagger$ is used to perform unbinding.  One can then query a bound representation, asking ``what was red?'' by unbinding ``red cat and yellow dog'' $\bind$ ``red''$^\dagger$  to get a vector $\approx$ ``cat'', where the resulting vector is necessarily corrupted by the noise by combining multiple vectors into a single fixed size representation. 
\par 
To perform this symbolic manipulation the binding operation can be defined as $\mathcal{B} = \mathbf{x} \bind \mathbf{y} =  \mathcal{F}^{-1}(\mathcal{F}(\mathbf{x}_i) \odot \mathcal{F}(\mathbf{y}_i))$, where $\mathcal {F}$ denotes the FFT and $\odot$ an element-wise multiplication\footnote{This is faster than an equivalent reformulation as multiplication by a circulant matrix of only real values.}. The inversion is defined as $\mathbf{y}^{\dagger} = \mathcal{F}^{-1} \left( \frac{1}{\mathcal{F}(\mathbf{y})} \right)$. 
\par 
Combined Plate showed that the response $\mathcal{B}^\top \mathbf{y}^\dagger$ should be $\approx 1$ if the vector $\mathbf{y} \in \mathcal{B}$, and $\approx 0$ if not present. These properties hold in expectation provided that all vectors satisfy the sufficient condition that their elements are I.I.D. sampled from a Gaussian with zero mean and variance $1/H$, where $H$ is the dimension of the vectors.
\par 
We will now show how to apply the same general logic of attention using HRR operations, creating an alternative (but not mathematically equivalent) form of self-attention that runs in linear time with respect to the sequence length. This is a slight ``abuse'' of the HRR, as our vectors will not be I.I.D. sampled random values, but results from prior layers in the network. Our design circumvents this issue in practice, which we will discuss shortly. We note this is a satisfying, but not required condition. Deviating from this adds more noise (our vectors are the outputs of prior layers in the network), but a softmax operation will act as a cleanup step to work without this condition.

Attention
can be represented using queries $\mathbf{Q}$, keys $\mathbf{K}$, and values $\mathbf{V}$ matrices where the final output is computed as the weighted sum of the values. A query vector can be mapped to a set of linked key-value pairs to retrieve the value vector associated with the associated key. The concept of \emph{binding} and \emph{unbinding} operations of HRR is applied to link the key-value pair (i.e.,
bind the terms together), and then query a single representation of all key-value pairs to find the response values. 
For this reason, we will define the steps in an element-by-element manner that more naturally corresponds to the HRR operations, but our implementation will work in a batched manner. For this reason, we will discuss a single query $\boldsymbol{q}_t \in \mathbb{R}^H$, against the set of $T$ key/value pairs $\boldsymbol{k}_t, \boldsymbol{v}_t \in \mathbb{R}^H$, where $H$ is the dimension of the representation and $t \in 1, 2, \cdots T$. Thus $\boldsymbol{K} = \left[\boldsymbol{k}_1, \boldsymbol{k}_2, \ldots \boldsymbol{k}_T \right]$ is a matrix of shape $(T, H)$, and similar for $\boldsymbol{Q}$ and $\boldsymbol{V}$. 
\par

First, we will create a superposition $\boldsymbol{\beta} \in \mathbb{R}^H$ of the key-value pairs, meaning that all vectors entering the superposition $\boldsymbol{\beta}$ are also similar (to some degree) to the final result. This is done by binding ($\bind$) each key-value pair to associate them, and summing the results to form the superposition:
\begin{equation}
    \boldsymbol{\beta} = \sum_{i=1}^T \boldsymbol{k}_i \bind \boldsymbol{v}_i
\end{equation}
$\boldsymbol{\beta}$ lets us compute interaction effects against all key-value pairs in one $\mathcal{O}(T H \log H)$ operation, avoiding the $\mathcal{O}(T^2 H)$ cost of explicit cross-correlation.

This now gives us a single vector $\boldsymbol{\beta}$ that represents the entire sequence of $T$ different key-value pair bindings. Now for each query we are interested in, we can obtain
a vector that approximately matches the values $\boldsymbol{v}_{1, 2, \ldots, T}$ via the symbolic property of HRRs that $\boldsymbol{x}^\dagger \bind \left(\boldsymbol{x} \bind \boldsymbol{y} + \boldsymbol{a} \bind \boldsymbol{b}\right) \approx \boldsymbol{y}$, giving:
\begin{equation}
    \boldsymbol{\hat{v}}_t = {\boldsymbol{q}_t}^\dagger \bind \boldsymbol{\beta}
\end{equation}
The queries are checked against the representation of all key-value pairs $\boldsymbol{\beta}$, where each $\boldsymbol{q}_t$ will contribute a corresponding value based on the response of the bound key, and the HRR framework allows us to perform them jointly. 
This now gives us a representation $\boldsymbol{\hat{v}}_t \in \mathbb{R}^H$ that represents the set of values present given the keys that respond to the input queries. 
We can then approximately determine the values present using the dot-product test that present values should result in $\approx 1$ scalars, performing:
\begin{equation}
\label{eq:cos_sim}
a_t  = \operatorname{cosine-similarity}\left(\boldsymbol{v}_t, \boldsymbol{\hat{v}}_t\right)
\end{equation}
Each $a_t$ is a scalar given the match between the original value $\boldsymbol{v}_t$ against the HRR extracted $\boldsymbol{\hat{v}}_t$, and is repeated for all $T$ values to give us a response on the relative magnitude of each value present.
With these approximate responses, we can compute a weighted distribution $\boldsymbol{w} \in \mathbb{R}^T$ by computing the softmax over all $a_{1, 2, \ldots, T}$ responses, giving $\boldsymbol{w} = \operatorname{softmax}(a_1, a_2, \ldots, a_T)$\footnote{We find  no meaningful difference in results  when using a temperature $\operatorname{softmax}(\exp(\alpha) [a_1, \ldots, a_T] )$.}. While each $a_t$ will be highly noisy due to the inherent noise of HRR's superposition $\boldsymbol{\beta}$, and an amplified level of noise due to the use of non-I.I.D. Gaussian elements, the softmax has the practical effect of removing this noise for us. This occurs because the HRR results in similar magnitude noise across each $a_t$, and the softmax operation is invariant to constant additions to all elements. 

For notational convenience to express this in more detail, let $\tilde{\Pi}_h(\boldsymbol{x}_1, \ldots, \boldsymbol{x}_k )$ denote the pairwise interactions of the $h$'th term in evaluating an expression of the form $\left(\sum_{i=1}^T \boldsymbol{x}_{i} \bind \boldsymbol{x}_{i+T}\right)^\top \boldsymbol{q}^T$, where all bold symbols are $H$ dimensional vectors. The response of any query of the form $\boldsymbol{q} = \boldsymbol{x}_m+\boldsymbol{z}$ takes the form 
$\frac{\sum_{h=1}^H (x_{m,h} + z_h) \tilde{\Pi}_h(\boldsymbol{x}_1, \ldots, \boldsymbol{x}_k) (-1)^{h+1}}{\left(\sum_{h=1}^H (-1)^{h+1} x_{m,h}  + \sum_{h=1}^H (-1)^{h+1} z_{h} \right) \left(\sum_{h=1}^H x_{m,h} + z_h \right)}$. In doing so we see that any noise vector $\boldsymbol{z}$ has a similar magnitude impact regardless of the target vector $\boldsymbol{x}_m$. Because the softmax is invariant to uniform magnitude adjustments to all inputs, and we have the same noise occurring for each computation, we get the behavior of the softmax effectively denoising the response due to the magnitude impacts. 
We discuss this further in \autoref{sec:softmax_denoise}. 
\par 
This softmax-based cleanup step is necessary because attempting to use $\boldsymbol{\hat{v}}_t$ directly results in degenerate random-guessing performance due to the noise of the HRR steps.
With $\boldsymbol{w}$ in hand, we obtain the final Attention result
\begin{equation}
\label{eq:attention}
\text{Attention}(\mathbf{Q}, \mathbf{K},  \mathbf{V}) = \left[{w}_1 \boldsymbol{v}_1, {w}_2 \boldsymbol{v}_2, \ldots, {w}_T \boldsymbol{v}_T, \right]
\end{equation}
returning a weighted version of the original values $\boldsymbol{V}$, approximating the standard attention's response. Critically, this process is linear in $T$ and approximates an all pairs interaction between queries and keys, as shown by \autoref{thm:cross}

The rest of self-attention works in the same manner as the standard Transformer. The Attention function's inputs and outputs are altered by linear layers, and instead of performing single attention, we split the feature vector $H$ of the query, key, and value into $h$ heads each having a feature size of $H'=H/h$. The attention is computed in parallel in each head and then merged into single attention which is projected to get the final output. The Hrrformer is implemented using JAX and a code snippet of the self-attention mechanism is presented in~\autoref{a:appendix}. The block diagram representation of the Hrrformer self-attention is presented in \autoref{fig:block}. The diagram is shown for single head and single batch elements for brevity. A high-level overview of the architecture in a multi-head setting is presented in \autoref{fig:mha} showing the analogy between Hrrformer and Transformer.

\begin{figure}[!ht]
\centerline{\includegraphics[width=\columnwidth]{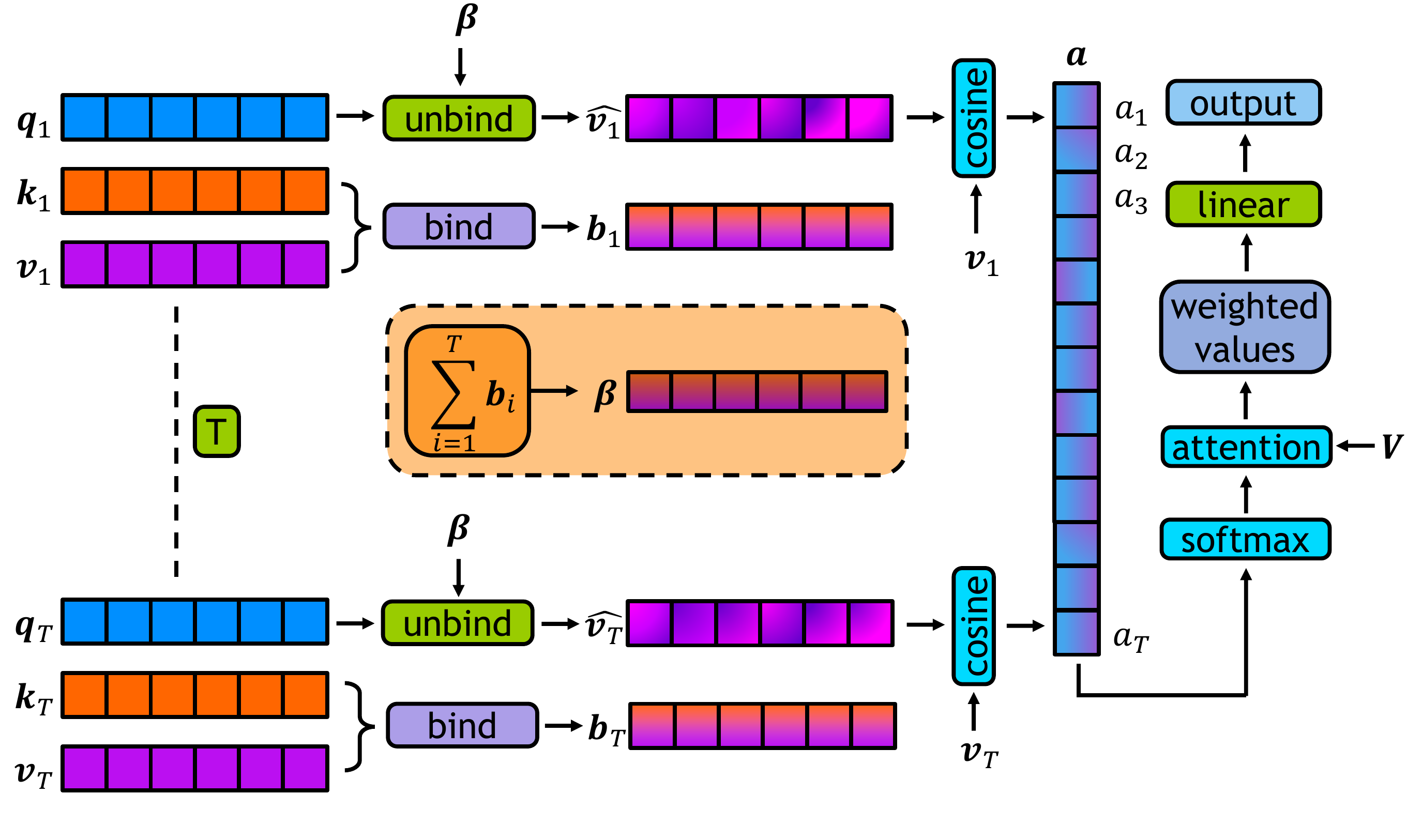}}
\caption{The block diagram of the Hrrformer self-attention. The \emph{dashed} straight line represents the continuation of the same process for each $T$ element. After computing the cosine similarity score vector $\mathbf{a}$, softmax is applied to compute the final attention weights $\mathbf{w}$ which is elementwise multiplied with value matrix $\boldsymbol{V} = \left[\boldsymbol{v}_1, \boldsymbol{v}_2, \ldots \boldsymbol{v}_T \right]$. Afterward, a linear layer is used to get the final output.}
\label{fig:block}
\end{figure}

\begin{figure}[!ht]
\centerline{\includegraphics[width=\columnwidth]{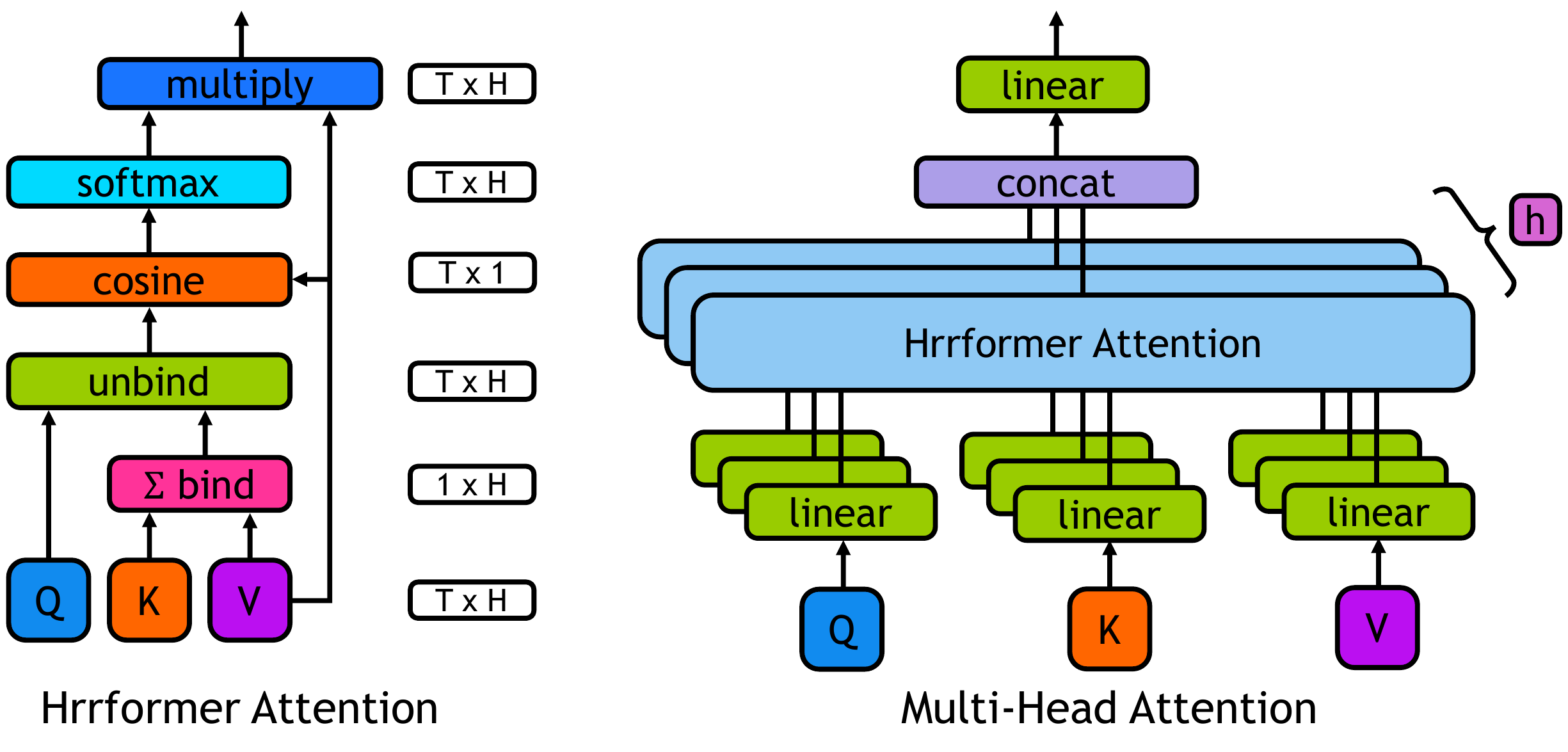}}
\caption{A high-level overview of our architecture, showing how the Hrrformer is analogous to the traditional transformer. Dataflow in a single-head with the shape of the tensor in different stages is shown on the left and multi-head attention is shown in right.}
\label{fig:mha}
\end{figure}

The time complexity of the binding/unbinding operation is $\mathcal{O}(H \log{H})$, which is performed $T$ times as the dominant cost. Therefore, the time and space complexity of the Hrrformer attention per layer is linear in sequence length $T$ where the time complexity is $\mathcal{O}(T H \log{H})$ and the space complexity is $\mathcal{O}(T H)$. 
\par 
This simple approach allows us to have fully replicated the same overall logical goals and construction of the attention mechanism first proposed by \cite{transformer}. The correspondence is not exact (e.g., returning weight original values instead of approximate value constructions), but allows us to avoid the non-I.I.D. issue of using arbitrary $\boldsymbol{Q}$, $\boldsymbol{K}$, and $\boldsymbol{V}$ as learned by the network. This neuro-symbolic reconstruction yields several benefits, as we will demonstrate in the next section. Simply replacing the self-attention in a standard Transformer with our HRR-based self-attention gives the ``Hrrformer'' that we will use to judge the utility of this new derivation.

\section{Experiments and Results} \label{sec:experiments}
The proposed Hrrformer is designed as an inexpensive alternative to the self-attention models for longer sequences. Experiments are performed to validate the effectiveness of the method in terms of time and space complexity in known benchmarks. 

Our first result is running many of the current popular and state-of-the-art (SOTA) xformers on the real-world classification task of the Ember malware detection dataset~\cite{ember}. This provides an example where the need to handle ever longer sequences exists and demonstrates that Hrrformer is one of the fastest and most accurate options on a problem with complex real-world dynamics. In doing so we also show that current SOTA methods such as Luna-256 do not generalize as well to new problem spaces, as our Hrrformer does.

Our second result will use the Long Range Arena (LRA)~\cite{lra} which has become a standard for evaluations in this space. The primary value of these results is to compare our Hrrformer with numerous prior works, establishing the broad benefits of faster time per epoch, convergence in $10\times$ fewer epochs, requiring only a single layer, and competitive overall accuracy. In addition, the LRA results are more accessible to the broader ML comunity and allow us to show visual evidence of HRR based attention learning to recover complex structure from a one-dimensional sequence.

\subsection{EMBER}
EMBER is a benchmark dataset for the malware classification task~\cite{ember}. The benchmark contains $600K$ labeled training samples ($300K$ malicious, $300K$ benign) and $200K$ labeled test samples ($100K$ malicious, $100K$ benign). The maximum sequence length of this dataset is over $100M$ which is not feasible for any of the self-attention models to train with. We experiment with relatively shorter sequence lengths starting from $T=256$ and doubling up to  $T=131072$ by truncating or padding the bytes until this maximum length is reached.
\par 
In this benchmark, Hrrformer is compared with Transformer~\cite{transformer}, H-Transformer-1D~\cite{zhu-soricut-2021-h}, Luna-256~\cite{Ma2021}, Performer~\cite{performers}, Linformer~\cite{linformer}, and F-Net~\cite{Lee-Thorp2021}. All use $8$ heads of a single encoder with $256$ embedding size and $512$ hidden size of the feed-forward network. Because this is a binary classification task, the encoder output is mapped into 2 logits output using back-to-back dense layers with ReLU activation. During training, the softmax cross-entropy loss function is optimized.
\par 
For sequence length $256$, the batch size is set to be $256$. In the experiment, as the sequence length doubles, we halved the batch size to fit the data and the model to the memory which can be expressed as $\max(2^{16 - \log_{2}{T}},~1)$. This is done to push other models to the maximum possible length, and keep the batch size consistent between experiments. Additionally, a timeout limit of $10,000$s per epoch is set before experiments are terminated. The dropout rate is chosen to be $0.1$, the learning rate is $10^{-3}$ with an exponential decay rate of $0.85$. Each of the models is trained for a total of 10 epochs in $16$ NVIDIA TESLA PH402 32GB GPUs.

Figure~\ref{fig:ember_acc} shows the classification accuracy of each of the methods for incremental sequence length from $512$ to $131072$. As the sequence length increases, Hrrformer outperforms the rest of the models achieving the highest $91.03\%$ accuracy for maximum sequence length $16384$. In terms of execution time F-Net is the only model that is faster than ours, however the accuracy of F-Net is an absolute $4.53\%$ points lower (\cref{tab:lra}).
Even after exponentially decaying batch size, we could not fit the standard Transformer model to the memory for the sequence length $8196$ indicating out-of-memory (OOM) in all figures. H-transformer-1d and Luna-256 crossed the timeout limit for sequence length $16384$ indicated out-of-time (OOT) in the figure. The detailed numeric results are presented in~\autoref{b:appendix} with additional results for the sequence length of $256$. The execution time for linear time complexity methods seems quadratic in the figure; this is due to the exponential decay of the batch size with the increase of sequence length, which was necessary to push  each model to its maximum possible  sequence length.  The more detailed timing information can be seen in \autoref{fig:ember_time}, where all models but F-Net and Hrrformer run out of time or memory before reaching the maximum sequence length. Note as well that as the sequence length increases, the already small difference in runtime between F-Net and Hrrformer reduces to near-zero. 

\begin{figure}[!ht]
\centerline{\includegraphics[width=\columnwidth]{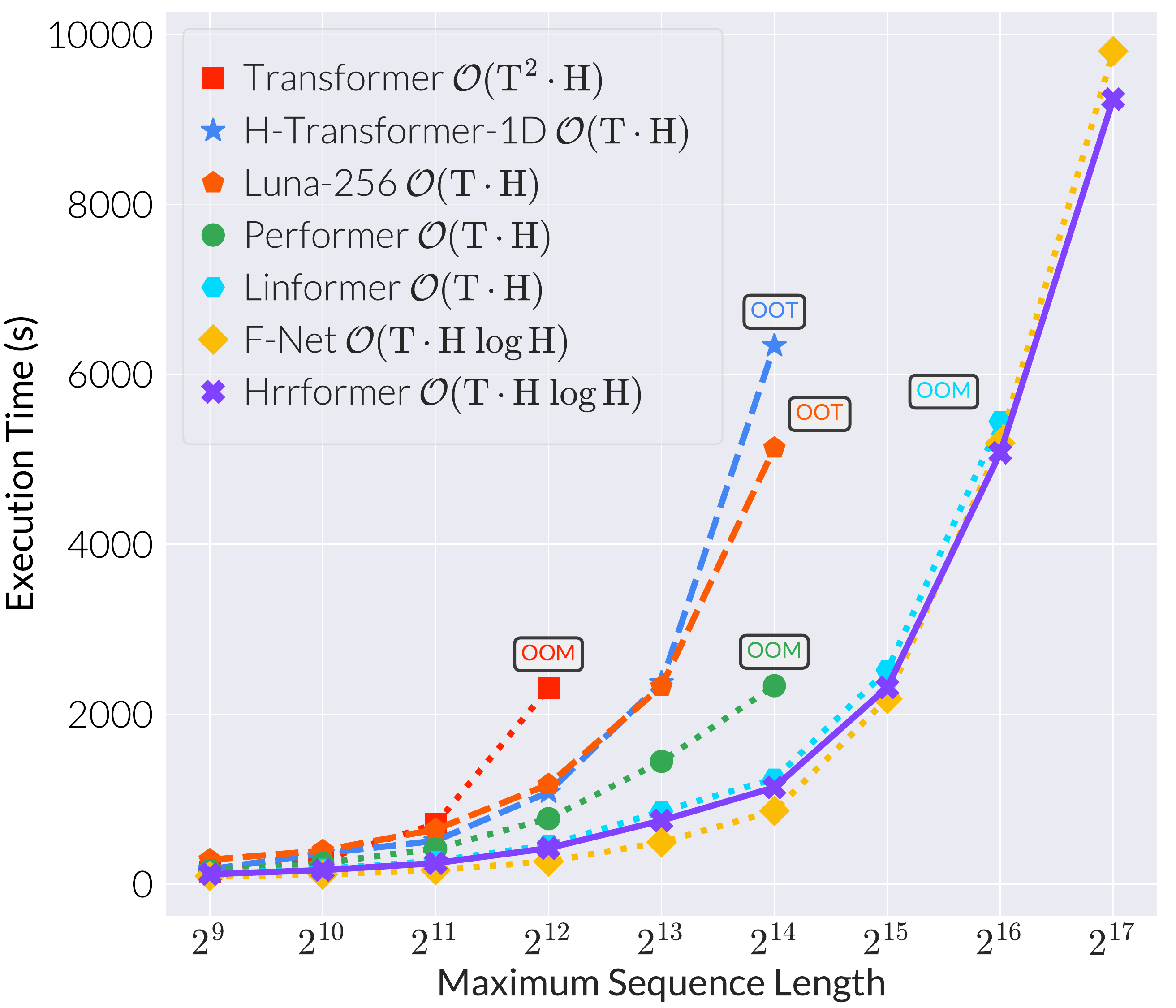}}
\caption{The total runtime on the Ember dataset for each algorithm, with the big-$\mathcal{O}$ runtime complexity associated. While Hrrformer is technically a slower big-$\mathcal{O}$ due to the extra $\log H$ term, the hidden size of the network is generally fixed and smaller than the sequence length. Thus we see in practice our design allows for faster execution in training and inference. Most prior methods fail early by running Out Of Memory (OOM) or Time (OOT).} 
\label{fig:ember_time}
\end{figure}

\par 
Of significant importance to our results is that Luna-256 performs considerably worse than all other options, compared to its top accuracy in the LRA. We hypothesize that the Ember task requires more complex reasoning and feature extraction over time and because Luna performs aggressive compression and approximation of the time component of the model it suffers in terms of accuracy. Our Hrrformer on the other hand has consistent behavior across Ember and the LRA: high accuracy, able to handle longer sequences, and convergence in few epochs, a requirement for working on this dataset which is 1 TB in size and is otherwise prohibitive in its scale.

\subsection{Long Range Arena} 
The Long Range Arena (LRA)~\cite{lra} benchmark comprises 6 diverse tasks covering image, text, math, language, and spatial modeling under long context scenarios ranging from $1K$ to $16K$. \textbf{ListOps} – task inspects the capability of modeling hierarchically structured data in a longer sequence context with mathematical operators \texttt{MAX}, \texttt{MEAN}, \texttt{MEDIAN}, and \texttt{SUM MOD} enclosed by delimiters. This is a ten-way classification problem with a maximum sequence length of $2K$. \textbf{Text} – is a byte/character level classification task using the IMDB movie review~\cite{b2} dataset. Character-level language modeling makes the models reason with compositional unsegmented data.%
This is a binary classification task with a maximum sequence length of $4K$. \textbf{Retrieval} – evaluates the model's ability to encode and compress useful information for matching and retrieval by modeling similarity score between two documents. For this task, the ACL Anthology Network~\cite{b3} dataset is used in a character level setup. This task has a maximum sequence length of $8K$ and this is a binary classification task. \textbf{Image} – is an image classification task of $10$ classes that uses grayscale CIFAR-10 dataset in a sequence of length $32 \times 32 = 1024$. This task allows assessing the model’s ability to process discrete symbols. \textbf{Pathfinder} – task evaluates the model’s performance over long-range spatial dependency. This is a binary classification task that classifies whether two circles are connected by a line which is introduced in~\cite{b4}, and includes distractor paths. The images have dimension $32\times32$ which is reshaped into $1024$. \textbf{Path-X} - is extremely difficult version of pathfinder task which contains images of dimension $128 \times 128 = 16384$ with additional distractor paths.

\begin{table*}[!ht]
\centering
\caption{Accuracy results of Hrrformer on Long Range Arena (LRA) benchmark. Even using just one layer Hrrformer is highly competitive, and the only method besides Luna is a Pareto improvement over the original Transformer. Our method is further advantaged in that it requires $10\times$ fewer epochs to reach competitive accuracies. Best results in \textbf{bold}, second best in \textit{italics}. }
\vspace{5pt}
\renewcommand{\arraystretch}{1.2}
\label{tab:lra}
\adjustbox{max width=\textwidth}{
\begin{tabular}{@{}lcccccccc@{}}
\toprule
Model                                      & ListOps (2k)        & Text (4k)           & Retrieval (4k)      & Image (1k)          & Path (1k)          & Path-X (16k)         & Avg          & Epochs      \\ \midrule
Transformer~\cite{transformer}             & 36.37          & 64.27          & 57.46          & 42.44          & 71.40          & FAIL               & 54.39          & 200         \\ \midrule
Local Attention~\cite{lra}                 & 15.82          & 52.98          & 53.39          & 41.46          & 66.63          & FAIL               & 46.06          & 200         \\
Linear Transformer~\cite{linear-trans}     & 16.13          & 65.90          & 53.09          & 42.34          & 75.30          & FAIL               & 50.55          & 200         \\
Reformer~\cite{reformer}                   & 37.27          & 56.10          & 53.40          & 38.07          & 68.50          & FAIL               & 50.67          & 200         \\
Sparse Transformer~\cite{sparse}           & 17.07          & 63.58          & 59.59          & 44.24          & 71.71          & FAIL               & 51.24          & 200         \\
Sinkhorn Transformer~\cite{sinkhorn}       & 33.67          & 61.20          & 53.83          & 41.23          & 67.45          & FAIL               & 51.29          & 200         \\
Linformer~\cite{linformer}                 & 35.70          & 53.94          & 52.27          & 38.56          & 76.34          & FAIL               & 51.36          & 200         \\
Performer~\cite{performers}                & 18.01          & 65.40          & 53.82          & 42.77          & \underline{77.05}    & FAIL               & 51.41              & 200         \\
Synthesizer~\cite{synthesizer}             & 36.99          & 61.68          & 54.67          & 41.61          & 69.45          & FAIL               & 52.88          & 200         \\
Longformer~\cite{longformer}               & 35.63          & 62.85          & 56.89          & 42.22          & 69.71          & FAIL               & 53.46          & 200         \\
BigBird~\cite{bigbird}                     & 36.05          & 64.02          & 59.29          & 40.83          & 74.87          & FAIL               & 55.01          & 200         \\
F-Net~\cite{Lee-Thorp2021}                 & 35.33          & 65.11          & 59.61          & 38.67          & \textit{77.78} & FAIL               & 54.42          & 200         \\
Nystromformer~\cite{Xiong2021}             & 37.15          & 65.52          & \textbf{79.56} & 41.58          & 70.94          & FAIL               & 58.95          & 200         \\
Luna-256~\cite{Ma2021}                     & 37.98          & 65.78          & \underline{79.56}    &  47.86    &  \textbf{78.55}     & FAIL           & \textbf{61.95} & 200         \\
H-Transformer-1D~\cite{zhu-soricut-2021-h} & \textbf{49.53} & \textbf{78.69} & 63.99          & 46.05          & 68.78          & FAIL               &  \underline{61.41}   & 200         \\ \midrule
Hrrformer Single-layer & 38.79 & \underline{66.50}  & 75.40 & \underline{48.47} & 70.71 & FAIL & 59.97 & \textbf{20} \\ 
Hrrformer Multi-layer & \underline{39.98} & 65.38 & 76.15 & \textbf{50.45} & 72.17 & FAIL & 60.83 & \textbf{20} \\ 
\bottomrule
\end{tabular}
}
\end{table*}

In Hrrformer, we use the same number or fewer parameters as mentioned in the LRA benchmark~\cite{lra} across the tasks and a list of hyper-parameters used in each task is provided in \autoref{b:appendix}. Global average pooling is applied to the output of the encoder sequences and subsequently back to back dense layers are used with ReLU activation to get the final logits output. During training, the softmax cross-entropy loss function is optimized using the Adam optimizer. We use the exponential decay learning rate with the initial value of $10^{-3}$, and the final value of $10^{-5}$. For all the tasks, Hrrformer is trained for a total of $20$ epochs both in the case of single- and multi-layer which is $10\times$ less training than previous works. The results in terms of accuracy in all the tasks of the LRA benchmark are presented in Table~\ref{tab:lra}.~\footnote{The Pathfinder task as originally reported by \cite{lra} uses a ``hard'' version of the task, but the code provided defaults to an ``easy'' version. Most papers do not make clear which version of the task is evaluated, and the F-Net authors indicated in correspondence the ``easy'' version was used. Luna-256 used the hard version, and other authors have not yet reached back to us. On the easy version, Hrrformer gets 80.81\% in a single-layer and 80.77\% in the multi-layer, but we report the hard version in our table and assume others are using the hard version.} 
\par 
Ours is one of only two methods that improve accuracy upon the Transformer and consistently displayed higher performance in all the tasks. We show the performance for both single and multiple layers. In 3 of the 5 tasks (ListOps, Text, Image), Hrrformer achieves the second-best results using only $1$ layer of the encoder. For the Image classification task, it achieves the best results of $50.45\%$ accuracy using $3$ layers of the encoder. Moreover, Hrrformer requires $10\times$ fewer epochs than others to produce comparable or better results. Overall, the multi-layered Hrrformer produces the third-best result of $60.83\%$ in the benchmark.

\begin{figure}[!ht]
\centerline{\includegraphics[width=\columnwidth]{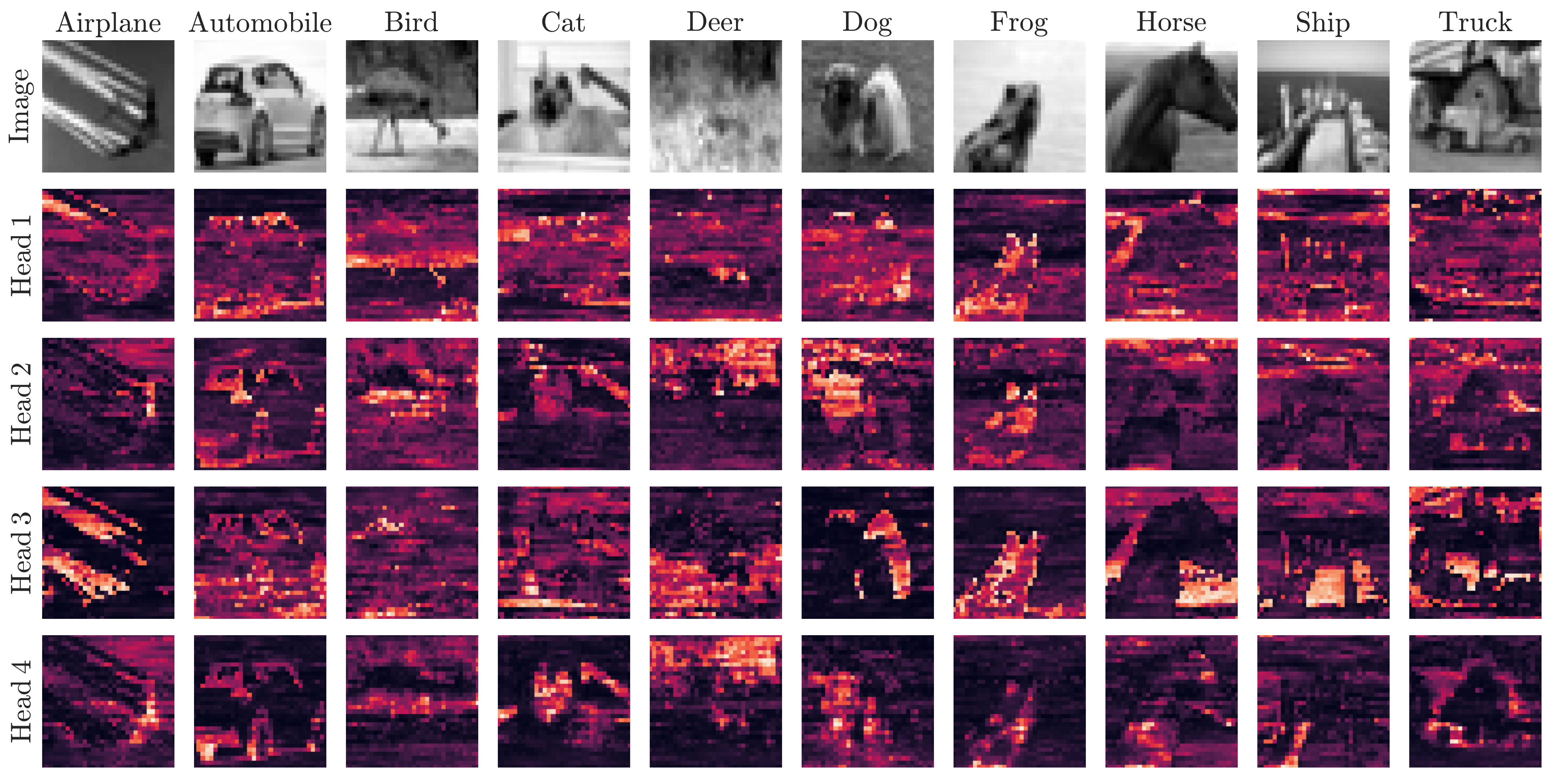}}
\caption{Visualization of weight vector $\mathbf{w} \in \mathbb{R}^{1024 \times 1}$ reshaped to $32 \times 32$, the shape of the original image of the CIFAR-10 dataset used in the LRA Image classification task. A single-layer Hrrformer is able to learn the 2D structure from the 1D sequence of the image. This is particularly noticeable in the Airplane, dog, Frog, and Horse images. Note context sensitive Head activation can be observed comparing Head 3 for dog vs Frog, where activation occurs for different pixel intensities indicating the model is not naively activating for simple color intensity.} 
\label{fig:heatmap}
\end{figure}

\begin{figure}[!ht]
\centerline{\includegraphics[width=\columnwidth]{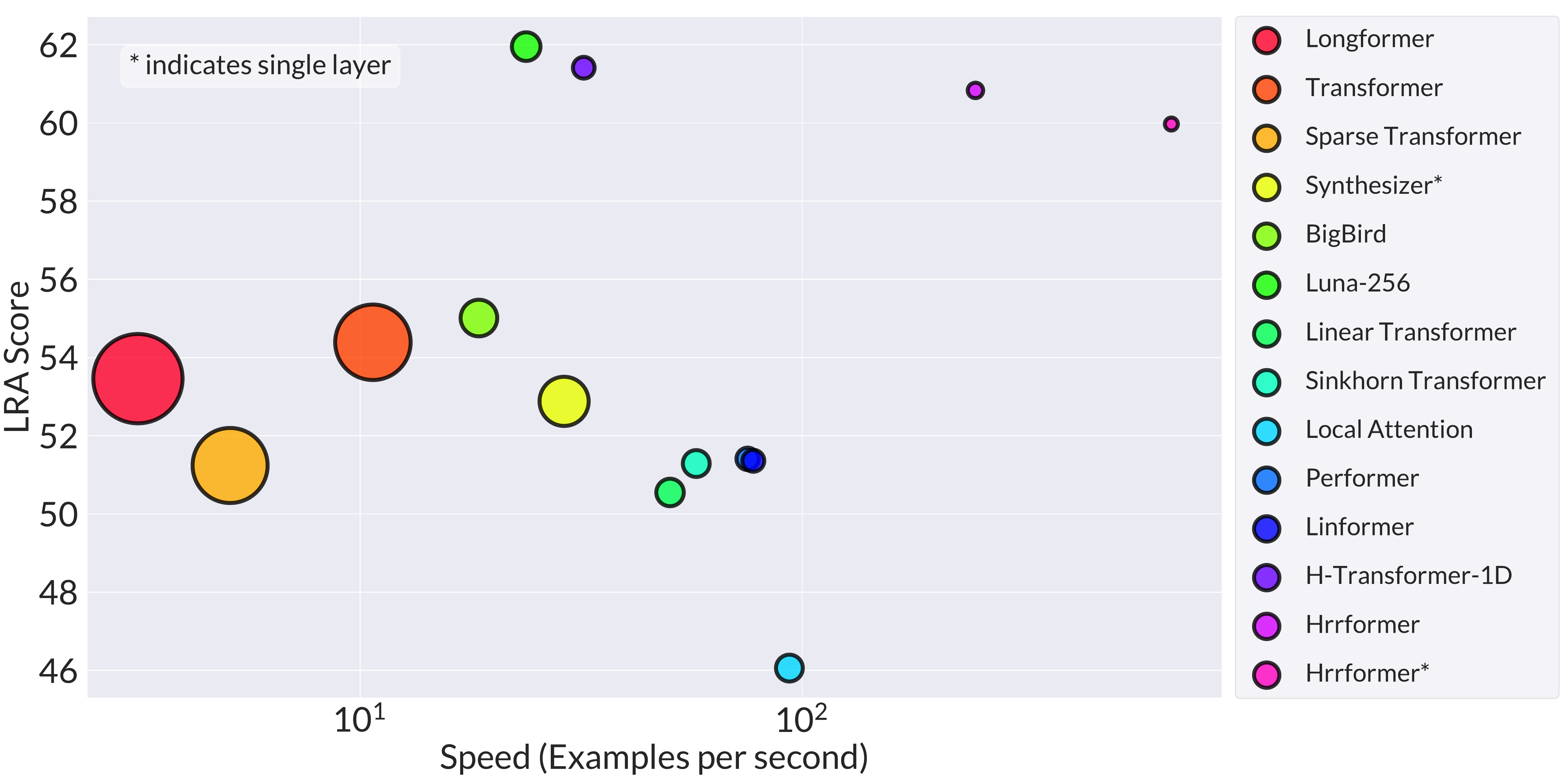}}
\caption{Performance ($y$-axis), Speed ($x$-axis, log-scale) of different xformers, and memory footprint on GPU are illustrated by the size of the circles. Hrrformer is in the top-right of the graph, with the smallest circle size, indicating it is the fastest and most memory efficient for training (this does \textit{not} factor in convergence speed).} 
\label{fig:lra_speed}
\end{figure}

The ability to learn with a single layer aids in both throughput and memory use. The result is surprising, and in visualizing the weight vector $\mathbf{w}$ we can confirm that a single layer is sufficient to learn the structure. We show this for the Image task of single-layer Hrrformer in \autoref{fig:heatmap} (multi-layer in \autoref{c:appendix}). Here, the weight vector $\mathbf{w} \in \mathbb{R}^{1024 \times 1}$ is reshaped to $32\times32$, the shape of the original grayscale images of the CIFAR-10 dataset for visualization. From the figure, it is clear that the Hrrformer is learning to identify the 2D structure from the 1D sequence of the Image classification task. We also compare against the standard Transformer in Appendix \autoref{fig:transformer_weight}, where it is less obvious how the model's weights might correspond to the 2D structure of the image. 
\par 
Hrrformer's benefits go beyond accuracy and convergence speed: it is fast and consumes the least amount of memory on GPU of the alternatives tested. Figure~\ref{fig:lra_speed} compares all the self-attention models in terms of LRA score, speed (training examples per second), and memory footprint (size of the circle). LRA score is the mean accuracy of all the tasks in the LRA benchmark. Speed and memory footprint is calculated on the byte-level text classification task per epoch. To measure these results, a single NVIDIA TESLA PH402 32GB GPU is utilized with a fixed batch size of $4$ and a maximum sequence length of $4000$ with an embedding size of $32$ and feature size of $64$. For all the models $6$ layers of the encoder are used. Both single- and multi-layered Hrrformer are $28\times$ and $10\times$ faster than the Luna-256~\cite{Ma2021} which has achieved the highest accuracy in the LRA benchmark. Hrrformer also consumes the least amount of memory, taking $79.15\%$ and $70.66\%$ less memory compared to Luna-256 in the case of single and multi-layered Hrrformer, respectively. The detailed numeric results of Figure~\ref{fig:lra_speed} are given in~\autoref{b:appendix}.

Hrrformer also reduces the amount of overfitting between training and test performance. We compare the training and test accuracy, and amount of overfitting of the Image classification task to the other self-attention models presented in LRA benchmark~\cite{lra} and for which data are available\footnote{We do not have the compute resources to run the other xformers on the LRA ourselves, in part due to the higher memory use that exceeds our infrastructure.}. Table~\ref{tab:overfitting} exhibits that the Hrrformer acquires the best results on the test set with an $6.83\%$ train/test gap. The learning curves of all the task is also presented in Appendix Figure~\ref{fig:lra_acc} demonstrating the lower overfitting nature of the Hrrformer across the tasks.

\begin{table}[!ht]
\centering
\caption{Training and test accuracy of different self-attention models on the Image classification task. Among all the models, Hrrformer achieves the best test accuracy with the least amount of overfitting (lower is better).}
\vspace{5pt}
\label{tab:overfitting}
\renewcommand{\arraystretch}{1.2}
\adjustbox{max width=\columnwidth}{
\begin{tabular}{@{}lccc@{}}
\hline
Model & Train Accuracy (\%) $\uparrow$ & Test Accuracy (\%) $\uparrow$ & Overfitting (\%) $\downarrow$ \\ \midrule
Transformer & 69.45 & 42.44 & 27.01 \\
Local Attention & 63.19 & 41.46 & 21.73 \\
Sparse Transformer & 66.74 & 44.24 & 22.50 \\
Longformer & 71.65 & 42.22 & 29.43 \\
Linformer & 97.23 & 38.56 & 58.67 \\
Reformer & 68.45 & 38.07 & 30.38 \\
Sinkhorn Transformer & 69.21 & 41.23 & 27.98 \\
Synthesizer & \textbf{97.31} & 41.61 & 55.70 \\
BigBird & 71.49 & 40.83 & 30.66 \\
Linear Transformer & 65.61 & 42.34 & 23.27 \\ 
Performer & 73.90 & 42.77 & 31.13 \\ \midrule
Hrrformer & 57.28 & \textbf{50.45} & \textbf{6.83} \\ \bottomrule
\end{tabular}%
}
\end{table}

Hrrformer's inference time is also faster than other options for long sequences. As an example, the time to make predictions for the text classification task is given in Appendix \autoref{tab:inference_results}, where the single-layer Hrrformer is the fastest option, followed by the multi-layer Hrrformer. We also find Hrrformer's inference time is relatively faster regardless of the batch size. The inference time for the Hrrformer with a batch size of 2 is still $5\times$ faster than the inference time for the Transformer with a batch size of 32. More details are presented in Appendix \autoref{tab:infer_batch}.

\section{Conclusion} \label{sec:conclusion}
The Hrrformer is a neuro-symbolic reconstruction of self-attention. The proposed method is faster in compute and consumes less memory per layer. We have tested Hrrformer on known LRA and EMBER benchmarks. In the LRA benchmark, Hrrformer has achieved the near state-of-the-art accuracy of $60.83\%$ using a single layer of an encoder. In terms of speed, it is $28\times$ and $10\times$ faster than the current SOTA in the case of single and multiple layers, respectively. Additionally, it takes $79.15\%$ and $70.66\%$ less memory on GPU compared to Luna-256 for single and multiple layers of Hrrformer. Moreover, it converges $10\times$ faster than other self-attention models. In the EMBER malware classification dataset, Hrrformer attained the highest $91.03\%$ accuracy for a maximum sequence length of $16384$ with a significantly faster processing rate. In conclusion, Hrrformer is $\approx 280\times$ faster to train and a single layer of the encoder is sufficient to learn the structure of the input.

\bibliography{refs, raffReferences}
\bibliographystyle{icml2023}

\newpage
\appendix
\onecolumn

\section{Self-attention Definition}
\label{a:appendix}
The code of the Hrrformer self-attention model is written in JAX. Below is a code snippet of the Multi-headed Hrrformer attention. The shape of the output vector of each line is given by a comment where $B$, $T$, and $H$ represent the batch size, maximum sequence length, and feature size, respectively. $h$ is the number of heads and $H'$ is the feature dimension in each head. 
\begin{figure*}[!htbp]
    \centering
\begin{lstlisting}[language=Python]
class SelfAttention(nn.Module):
    features: int
    heads: int

    def setup(self):
        self.binding = Binding()
        self.unbinding = Unbinding()
        self.similarity = CosineSimilarity()

    @nn.compact
    def __call__(self, inputs, mask=None):
        dense = partial(nn.DenseGeneral, features=self.features, use_bias=False)

        q, k, v = (dense(name='query')(inputs),  # (B, T, H)
                   dense(name='key')(inputs),    # (B, T, H)
                   dense(name='value')(inputs))  # (B, T, H)
                   
        q, k, v = (split(q, self.heads),  # (B, h, T, H')
                   split(k, self.heads),  # (B, h, T, H')
                   split(v, self.heads))  # (B, h, T, H')

        bind = self.binding(k, v, axis=-1)           # (B, h, T, H')
        bind = np.sum(bind, axis=-2, keepdims=True)  # (B, h, 1, H')

        vp = self.unbinding(bind, q, axis=-1)                   # (B, h, T, H')
        scale = self.similarity(v, vp, axis=-1, keepdims=True)  # (B, h, T, 1)

        if mask is not None: 
            scale = scale + (1. - mask) * (-1e9)  # (B, h, T, 1)
        weight = nn.softmax(scale, axis=-2)       # (B, h, T, 1)
        weighted_value = weight * v               # (B, h, 1, H')

        weighted_value = merge(weighted_value)         # (B, T, H)
        output = dense(name='output')(weighted_value)  # (B, T, H)
        return output
\end{lstlisting}
    \caption{Multi-headed Hrrformer Self-attention.}
\end{figure*}

\begin{theorem} \label{thm:cross}
The Hrrformer Attention approximates an all-pairs interaction between all queries and key-values. 
\end{theorem}

\begin{proof}
Expand \autoref{eq:cos_sim} as
$ \operatorname{cosine-sim}\left(\boldsymbol{v}_t, {\boldsymbol{q}_t}^\dagger \bind \left( \sum_{i=1}^T \boldsymbol{k}_i \bind \boldsymbol{v}_i  \right) \right) $. The distributive property of the  binding operation $\bind$ allows us to move the query inside summation, producing $\operatorname{cosine-sim}\left(\boldsymbol{v}_t, \sum_{i=1}^T  \boldsymbol{q}_i^\dagger \bind \boldsymbol{k}_i \bind \boldsymbol{v}_i    \right)$. At the cost of noise terms not specified, we can see that the response of the cosine similarity is produced from an interaction between the time step $t$ and summation of all query-key pairs for $1,2,\cdots,T$ steps, showing that a cross-product is approximated by the Hrrformer. 
\end{proof}

\section{Hyperparameters \& Numeric Results} \label{b:appendix}
The hyperparameters used in each task of the Long Range Arena (LRA) benchmark and EMBER malware classification task are presented in \autoref{tab:hyper_params}. In all of the tasks, the Adam optimizer is used with an exponential decay learning rate. The starting learning rate is $10^{-3}$ and the final learning rate is $10^{-5}$. The decay rate indicates the amount of learning rate decay per epoch. MLP dim indicates the number of features used in the first linear layer of the MLP block after the attention block.

\begin{table*}[!ht]
\centering
\caption{List of the hyperparameters used in the Long Range Arena (LRA) benchmark and EMBER malware classification task.}
\vspace{5pt} 
\label{tab:hyper_params}
\renewcommand{\arraystretch}{1.5}
\adjustbox{max width=\textwidth}{
\begin{tabular}{|c|c|c|c|c|c|c|c|c|c|c|}
\hline
\multirow{2}{*}{Task} & \multirow{2}{*}{\shortstack{Positional\\Embedding}} & \multirow{2}{*}{Batch size} & \multirow{2}{*}{Vocab size} & \multirow{2}{*}{\shortstack{Maximum\\Sequence Length}} & \multirow{2}{*}{Embed dim} & \multirow{2}{*}{MLP dim} & \multirow{2}{*}{Heads} & \multirow{2}{*}{Layers} & \multirow{2}{*}{Classes} & \multirow{2}{*}{\shortstack{Decay\\rate}} \\
 &  &  &  &  &  &  &  &  &  &  \\ \hline
ListOps & Learned & 32 & 17 & 2000 & 512 & 256 & 8 & 6 & 10 & 0.90 \\ \hline
Text & Fixed & 32 & 257 & 4000 & 512 & 1024 & 8 & 6 & 2 & 0.90 \\ \hline
Retrieval & Fixed & 64 & 257 & 4000 & 128 & 64 & 4 & 4 & 2 & 0.90 \\ \hline
Image & Fixed & 32 & 256 & 1024 & 256 & 128 & 4 & 3 & 10 & 0.95 \\ \hline
Path & Learned & 128 & 256 & 1024 & 1024 & 256 & 8 & 2 & 2 & 0.95 \\ \hline
Malware & Learned & $\max(2^{16 - \log_{2}{T}},~1)$ & 257 & $T$ & 256 & 512 & 8 & 1 & 2 & 0.85 \\ \hline
\end{tabular}
}
\end{table*}

The detailed results of the Hrrformer of Figures~\ref{fig:lra_speed} are presented here. The numerical results of the comparison of Hrrformer with other self-attention models in terms of LRA score, speed (examples per second), and memory footprint (size of the circle) are presented in Table~\ref{tab:lra_results}. From the table, it can be observed that the Hrrformer only lags $1.12\%$ behind Luna-256~\cite{Ma2021} in the LRA score. However, in terms of speed, single- and multi-layered Hrrformer are $28\times$ and $10\times$ faster than Luna-256. Moreover, Hrrformer consumes $79.15\%$ and $70.66\%$ less memory than Luna-256 in the case of single and multi-layered Hrrformer, respectively. The numerical results of EMBER malware classification are presented in Table~\ref{tab:malware_acc}. From the table, it can be observed that as the sequence length increases, Hrrformer surpasses the other models, and for the sequence length $16,384$,  has achieved the highest accuracy of $91.03\%$.

\begin{table*}[!ht]
\centering
\caption{LRA score, speed in examples per second, and total memory usage in MB of all the different xformer models used in LRA benchmark. The speed and memory usage metrics are computed using 6 layers of encoder in byte-level text classification task. In the chart, * indicates the use of single layer of encoder. Best results in \textbf{bold}, second best in \textbf{italics}. }
\vspace{5pt}
\renewcommand{\arraystretch}{1.2}
\label{tab:lra_results}
\adjustbox{max width=0.75\columnwidth}{
\begin{tabular}{lccccc}
\toprule
\multirow{2}{*}{Model} & \multirow{2}{*}{\shortstack{LRA Score\\Accuracy (\%)}} & \multirow{2}{*}{\shortstack{Speed\\(Examples per Second)}} & \multirow{2}{*}{Time (s)} & \multirow{2}{*}{\shortstack{Memory Usage\\(MB)}} \\
 &  &  &  &  \\ \midrule
Longformer & 53.46 & 3.14 & 7959.42 & 30978.21 \\
Sparse Transformer & 51.24 & 5.08 & 4923.98 & 21876.57 \\
Transformer & 54.39 & 10.68 & 2340.31 & 22134.52 \\
BigBird & 55.01 & 18.56 & 1347.26 & 5218.89 \\
Luna-256 & \textbf{61.95} & 23.74 & 1053.25 & 3184.66 \\
Synthesizer* & 52.88 & 28.92 & 864.45 & 9377.55 \\
H-Transformer-1D & \underline{61.41} & 32.03 & 780.42 & 1838.28 \\
Linear Transformer & 50.55 & 50.22 & 497.84 & 2941.39 \\
Sinkhorn Transformer & 51.29 & 57.56 & 434.31 & 2800.88 \\
Performer & 51.41 & 75.23 & 332.31 & 1951.53 \\
Linformer & 51.36 & 77.49 & 322.62 & 1867.64 \\
Local Attention & 46.06 & 93.51 & 267.35 & 2800.88 \\
Hrrformer & 60.83 & \underline{246.45} & \underline{101.44} & \underline{934.41} \\
Hrrformer* & 59.97 & \textbf{683.81} & \textbf{36.56} & \textbf{663.88} \\ \bottomrule
\end{tabular}%
}
\end{table*}

\begin{table*}[!ht]
\centering
\caption{Accuracy and the execution time of the different self-attention models for different sequence lengths in the EMBER malware classification dataset. Best results in \textbf{bold}. }
\label{tab:malware_acc}
\vspace{5pt}
\renewcommand{\arraystretch}{1.2}
\adjustbox{max width=\textwidth}{
\begin{tabular}{cccccccccccc}
\toprule
\multirow{2}{*}{Model} &  & \multicolumn{10}{c}{Maximum Sequence Length} \\ \cline{3-12} 
 &  & 256 & 512 & 1,024 & 2,048 & 4,096 & 8,192 & 16,384 & 32,768 & 65,536 & 131,072 \\ \midrule
\multirow{2}{*}{Transformer} & Accuracy (\%) & 74.87 & 84.27 & 86.74 & 87.76 & 88.43 & -- & -- & -- & -- & -- \\
 & Time (s) & 101.59 & 146.96 & 286.98 & 708.7 & 2305.28 & -- & -- & -- & -- & -- \\ \midrule
\multirow{2}{*}{H-Transformer-1D} & Accuracy (\%) & 59.59 & 78.17 & 85.45 & 87.8 & 90.14 & 88.9  & 90.48 & -- & -- & --\\
 & Time (s) & 116.6 & 175.04 & 362.41 & 509.63 & 1082.67 & 2371.96  & 6336.37 & -- & -- & -- \\ \midrule
 \multirow{2}{*}{Luna-256} & Accuracy (\%) & 70.21 & 74.8 & 77.01 & 80.06 & 79.18 & 83.76 & 83.55 & -- & -- & -- \\
 & Time (s) & 243.04 & 287.5 & 395.87 & 643.81 & 1172.35 & 2326.15 & 5132.95 & -- & -- & -- \\ \midrule
 \multirow{2}{*}{Performer} & Accuracy (\%) & 78.0 & \textbf{87.74} & \textbf{88.91} & \textbf{89.77} & 89.06 & 89.88 & 85.68 & -- & -- & -- \\
 & Time (s) & 115.77 & 159.59 & 247.02 & 418.1 & 770.75 & 1444.38 & 2334.94 & -- & -- & -- \\ \midrule
 \multirow{2}{*}{Linformer} & Accuracy (\%) & \textbf{79.52} & 86.41 & 88.73 & 88.25 & 86.57 & 86.53 & 86.94 & 85.70 & 83.75 & -- \\
 & Time (s) & 99.18 & 124.66 & 179.56 & 273.71 & 459.68 & 855.85 & 1239.88 & 2518.44 & 5445.57 & -- \\ \midrule
 \multirow{2}{*}{F-Net} & Accuracy (\%) & 76.42 & 80.25 & 80.87 & 84.34 & 83.55 & 86.36 & 86.00 & 86.29 & 86.45 & 86.40 \\
 & Time (s) & 84.84 & 95.58 & 113.2 & 165.77 & 267.21 & 492.44 & 861.48 & 2182.30 & 5191.26 & 9800.97 \\ \midrule
  \multirow{2}{*}{Hrrformer} & Accuracy (\%) & 78.06 & 83.95 & 88.07 & 89.22 & \textbf{90.59} & \textbf{90.89} & \textbf{91.03} & \textbf{90.65} & \textbf{90.13} & \textbf{89.46} \\
 & Time (s) & 91.35 & 117.96 & 165.18 & 247.32 & 423.55 & 748.48 & 1138.75 & 2315.62 & 5076.65 & 9237.78 \\ \midrule
\end{tabular}%
}
\end{table*}

In addition we provide the time to perform inference over the entire LRA text classification task for batch sizes varying between $2$ and $32$. This is shown in \autoref{tab:infer_batch}, where the time decreases as batch size increases due to reduced overhead and higher GPU compute efficiency. As can be seen the Hrrformer is uniformly faster, and more consistent in total run-time. Similarly, our method is faster for larger and small batch sizes, a particularly valuable benefit in inference where batching is not always possible. This can be seen in \autoref{tab:inference_results}, where the inference time for the Hrrformer with a batch size of $2$ is still $5\times$ faster than the inference time for the Transformer with a batch size of $32$.

\begin{figure}[!h]
\centerline{\includegraphics[width=\columnwidth]{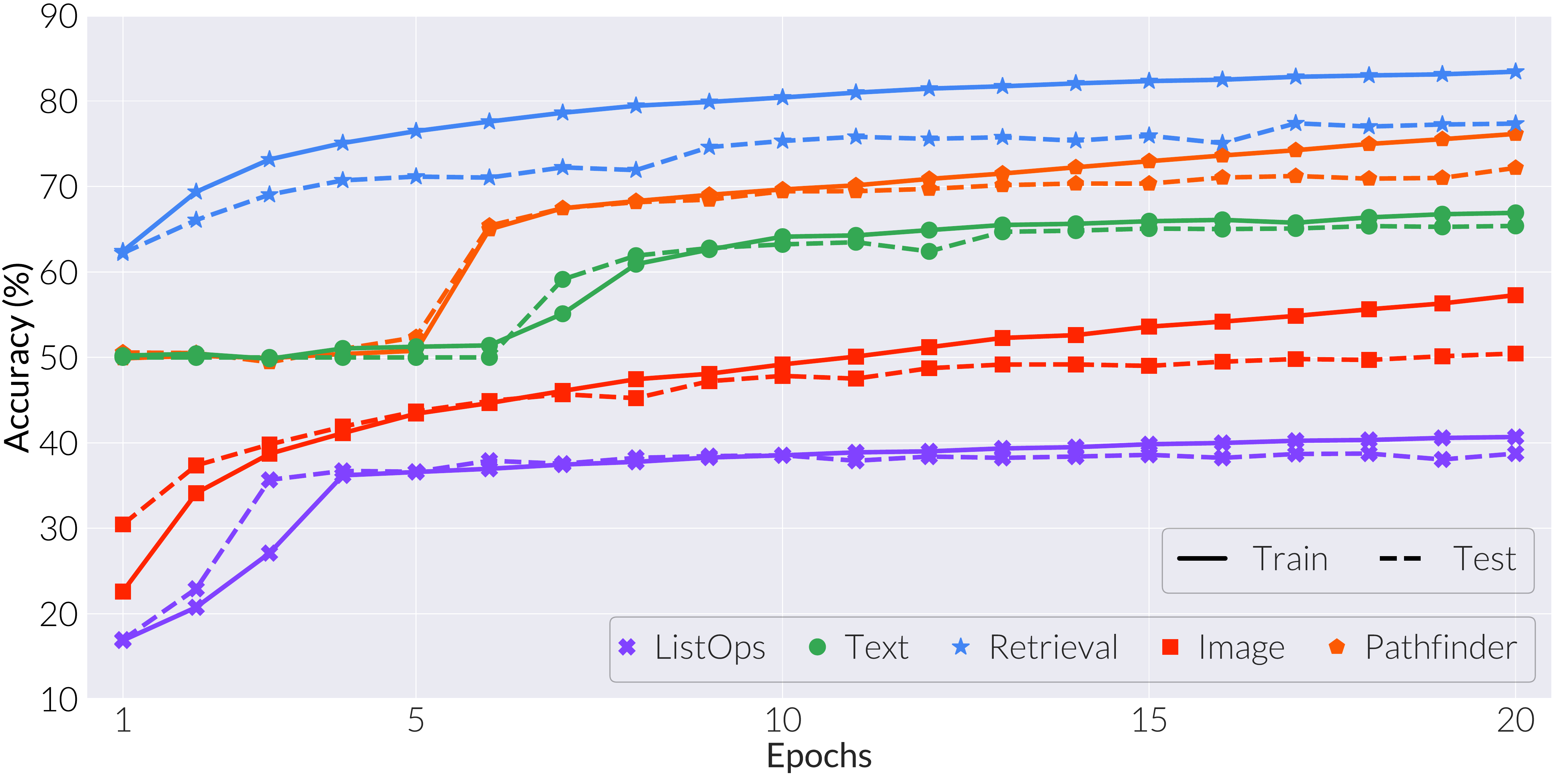}}
\caption{The learning curves of multi-layered Hrrformer in the LRA tasks. The training performance is solid lines and the test is dashed. Where prior works required 200 epochs of training, we can see that \textbf{20} epochs are sufficient for our Hrrformer. In most of the tasks, the 10-epoch performance of our Hrrformer is still highly competitive.} 
\label{fig:lra_acc}
\end{figure}

\begin{table}[!ht]
\centering
\caption{Inference timing comparison between single Hrrformer block and single Transformer block for different batch sizes (2-32). The experiment is performed on the LRA text classification task.}
\vspace{5pt}
\label{tab:infer_batch}
\adjustbox{max width=0.55\columnwidth}{%
\begin{tabular}{@{}ccccc@{}}
\toprule
\multirow{2}{*}{Batch Size} & \multicolumn{2}{c}{Hrrformer} & \multicolumn{2}{c}{Transformer} \\ \cmidrule(l){2-5} 
 & Time (s) & Memory (MB) & Time (s) & Memory (MB) \\ \midrule
 2 &  152.99 &  663.88 &  975.98 &  1584.53 \\
 3 &  127.34 &  936.51 &  815.30 &  4809.95 \\
 4 &  118.39 &  938.61 &  813.72 &  4809.95 \\
 5 &  117.15 &  1475.48 &  812.09 &  9104.92 \\
 6 &  115.37 &  1481.77 &  810.57 &  9107.01 \\
 7 &  115.44 &  1483.87 &  810.14 &  9109.11 \\
 8 &  113.01 &  1488.06 &  810.59 &  9109.11 \\
 9 &  114.81 &  2563.90 &  809.61 &  17701.14 \\
 10 &  113.34 &  2563.90 &  809.87 &  17701.14 \\
 11 &  113.83 &  2570.19 &  808.71 &  17705.34 \\
 12 &  113.11 &  2572.29 &  808.52 &  17705.34 \\
 13 &  114.65 &  2576.48 &  808.35 &  17707.43 \\
 14 &  114.64 &  2578.58 &  808.66 &  17709.53 \\
 15 &  114.42 &  2582.77 &  808.12 &  17711.63 \\
 16 &  113.81 &  2589.07 &  808.80 &  17711.63 \\
 17 &  86.80 &  2593.26 &  807.34 &  30976.11 \\
 18 &  85.95 &  4742.84 &  806.94 &  30976.11 \\
 19 &  85.56 &  4749.13 &  806.91 &  30978.21 \\
 20 &  85.11 &  4749.13 &  808.78 &  30980.31 \\
 21 &  84.78 &  4755.42 &  806.70 &  30980.31 \\
 22 &  83.95 &  4757.52 &  806.70 &  30982.41 \\
 23 &  83.23 &  4763.81 &  806.50 &  30986.60 \\
 24 &  81.84 &  4765.91 &  807.04 &  30988.70 \\
 25 &  83.06 &  4768.01 &  809.12 &  30988.70 \\
 26 &  83.01 &  4772.20 &  806.10 &  30990.79 \\
 27 &  82.87 &  4776.39 &  806.89 &  30992.89 \\
 28 &  82.70 &  4780.59 &  806.70 &  30994.99 \\
 29 &  82.60 &  4784.78 &  807.45 &  30994.99 \\
 30 &  82.30 &  4788.98 &  806.71 &  30999.18 \\
 31 &  82.44 &  4791.07 &  807.51 &  30999.18 \\
 32 &  80.83 & 4797.37 & 807.13 & 31001.28 \\ \bottomrule
\end{tabular}%
}
\end{table}

\begin{table}[!ht]
\centering
\caption{Inference time comparison with different self-attention models. The experiment is performed on the LRA text classification task with 6 layers of the encoder. In the chart, * indicates single layer.}
\vspace{5pt}
\label{tab:inference_results}
\adjustbox{max width=0.6\columnwidth}{
\begin{tabular}{@{}cccc@{}}
\toprule
Model & Time (s) $\downarrow$ & Speed (examples per second) $\uparrow$ & Memory (MB) $\downarrow$ \\ \midrule
Local Attention & 1910.33 & 13.09 & 9369.16 \\
Synthesizer & 1848.77 & 13.52 & 8983.28 \\
Sinkhorn Transformer & 1848.76 & 13.52 & 8983.28 \\
Transformer & 813.67 & 30.72 & 4805.75 \\
Sparse Transformer & 361.69 & 69.12 & 5229.38 \\
Longformer & 337.81 & 74.01 & 2815.56 \\
Performer & 170.75 & 146.41 & 728.89 \\
Linear Transformer & 163.15 & 153.23 & 913.44 \\
BigBird & 92.89 & 269.14 & 645.01 \\
Linformer & 88.96 & 281.03 & 645.01 \\
Hrrformer & \underline{33.38} & \underline{748.95} & \textbf{527.56} \\
Hrrformer* & \textbf{31.82} & \textbf{785.67} & \textbf{527.56} \\ \bottomrule
\end{tabular}%
}
\end{table}

\section{Weight Visualization}
\label{c:appendix}
The weight vector $\mathbf{w}$ 
is visualized for LRA image classification task. In this task, grayscale images of the CIFAR-10 dataset of dimension $32\times32$ are reshaped into a sequence of length $1024$. Therefore, the weight vector has the shape of $\mathbb{R}^{1024 \times 1}$. This vector is reshaped back to $32 \times 32$ for visualization which shows where in the image the weight vector of each head puts its attention. Figure~\ref{fig:heatmap_multi} demonstrates the attention map of the $4$ heads in each of the $3$ layers of Hrrformer for all the CIFAR-10 classes.

\begin{figure*}[!ht]
\centerline{\includegraphics[width=0.92\textwidth]{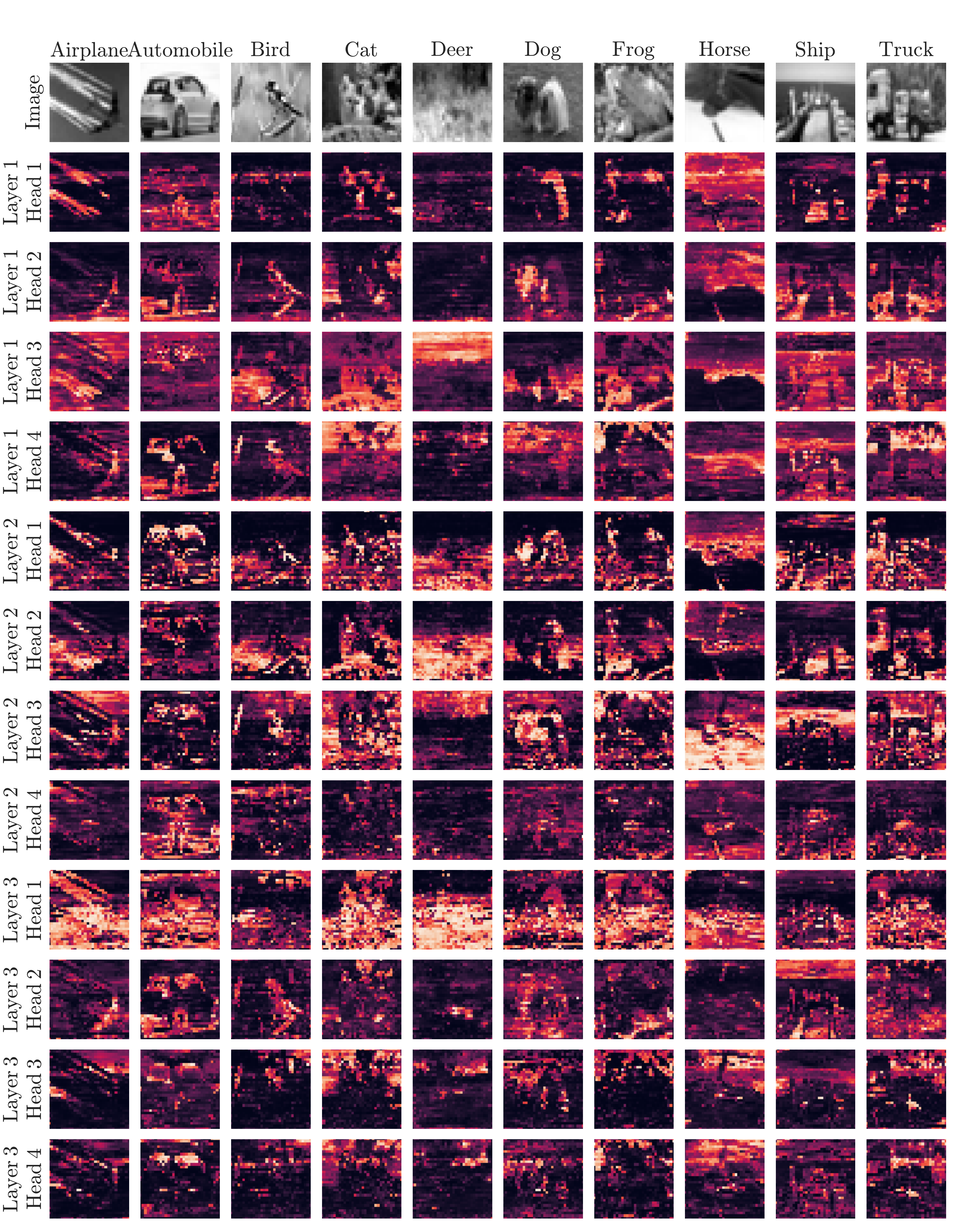}}
\caption{Visualization of weight $\mathbf{w} \in \mathbb{R}^{1024 \times 1}$ vector of multi-layer Hrrformer, reshaped to $32 \times 32$, the shape of the original image of the CIFAR-10 dataset used in the LRA image classification task.} 
\label{fig:heatmap_multi}
\vspace{5pt}
\end{figure*}

For the standard Transformer, the responses are a matrix of cross-correlations rather than a single vector. This makes the response more difficult to interpret. To visualize in the same manner we average the response of correlations with respect to a single item $t$ to get the same $1024$ shape, and visualize the results in \autoref{fig:transformer_weight}. As can be seen, the identification of structure is not as obviously. 

\begin{figure*}[!ht]
\centerline{\includegraphics[width=\textwidth]{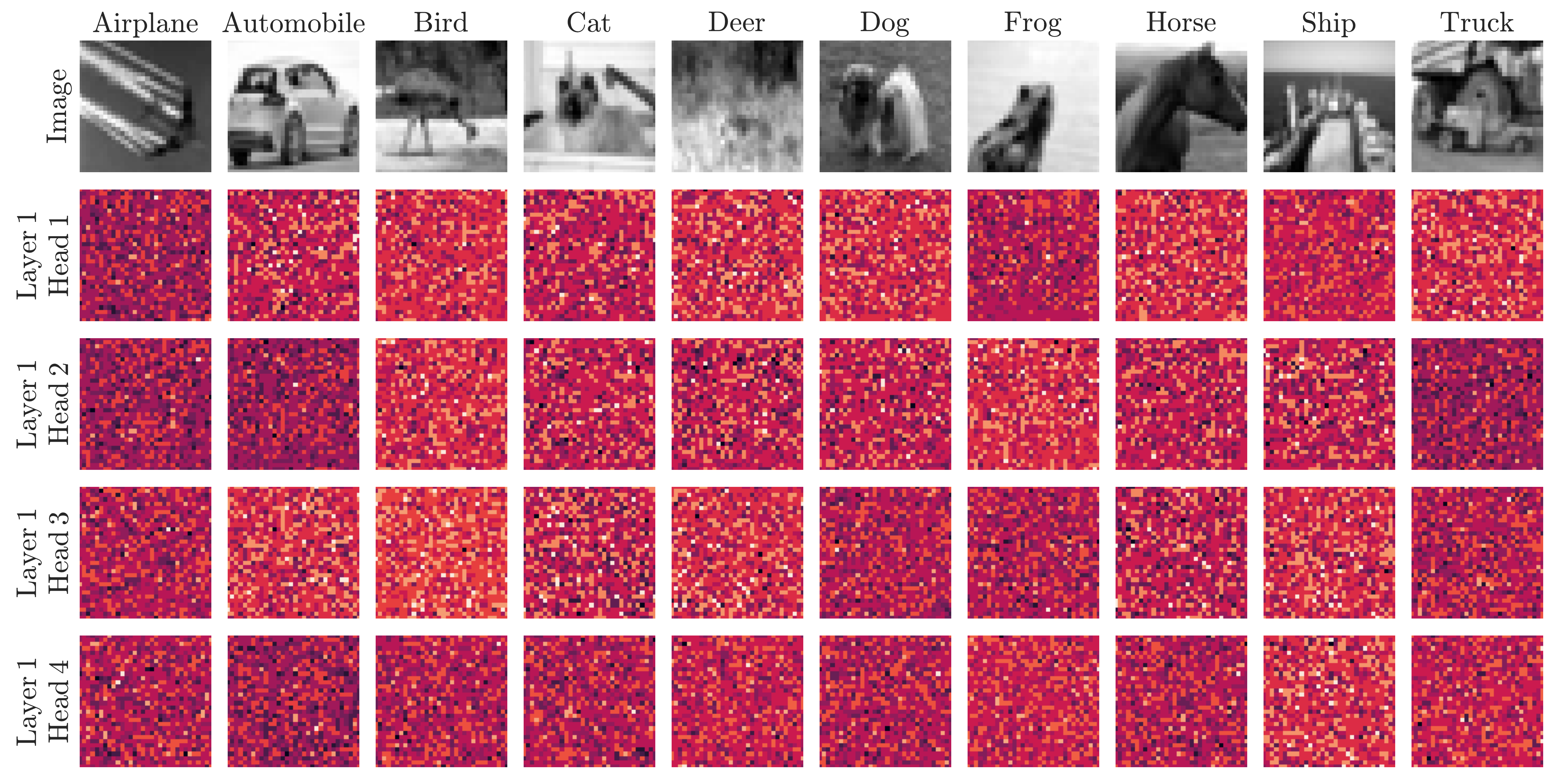}}
\caption{Visualization of transformer’s weight matrix of shape ${1024 \times 1024}$. The average attention across the key dimension is taken to reduce its dimension to $1024 \times 1$ and reshaped to $32 \times 32$, the shape of the original image of the CIFAR-10 dataset used in the LRA Image classification task.} 
\label{fig:transformer_weight}
\end{figure*}

\section{How Softmax ``Denoises'' Dot Product} \label{sec:softmax_denoise}
To understand how we can use the softmax operation as a kind of denoising step, consider the $H$ dimensional vectors $\boldsymbol{a}$, $\boldsymbol{b}$, $\boldsymbol{c}$, $\boldsymbol{d}$, and $\boldsymbol{z}$. If each element of all these vectors is sampled from $\mathcal{N}(0, 1/H)$, then we would expect that $(\boldsymbol{a} \otimes \boldsymbol{b} + \boldsymbol{c} \otimes \boldsymbol{d}  )^\top \boldsymbol{a}^\dagger \approx 1$. Similarly, the value $\boldsymbol{z}$ is not present, so we expect that $(\boldsymbol{a} \otimes \boldsymbol{b} + \boldsymbol{c} \otimes \boldsymbol{d}  )^\top \boldsymbol{z}^\dagger  \approx 0$. 
Now let us consider our use case, where the I.I.D. property is not true, and the query that is a noisy version of a present item. For simplicity of notation, we will use the explicit case of $H=2$ dimensions. We can query for $\boldsymbol{a}+\boldsymbol{z}$ get:
$$
\frac{\left(a_{0} + z_{0}\right) \left(a_{0} b_{0} + a_{1} b_{1} + c_{0} d_{0} + c_{1} d_{1}\right) - \left(a_{1} + z_{1}\right) \left(a_{0} b_{1} + a_{1} b_{0} + c_{0} d_{1} + c_{1} d_{0}\right)}{\left(a_{0} - a_{1} + z_{0} - z_{1}\right) \left(a_{0} + a_{1} + z_{0} + z_{1}\right)}
$$
Similarly if we query with $\boldsymbol{c}+\boldsymbol{z}$ we instead get:
$$
\frac{\left(c_{0} + z_{0}\right) \left(a_{0} b_{0} + a_{1} b_{1} + c_{0} d_{0} + c_{1} d_{1}\right) - \left(c_{1} + z_{1}\right) \left(a_{0} b_{1} + a_{1} b_{0} + c_{0} d_{1} + c_{1} d_{0}\right)}{\left(c_{0} - c_{1} + z_{0} - z_{1}\right) \left(c_{0} + c_{1} + z_{0} + z_{1}\right)}
$$

Notice that in both cases we have shared terms that are multiplied and added together. Under the sufficient conditions of I.I.D Gaussian, the linearity of expectation results in these terms canceling out into a single random variable with a zero mean. 

However, these also have the artifact in our application that for a non-present query, the response magnitude will have a similar value due to the repeated shared terms. 

We can simplify our understanding of this by imagining that there is an additional noise constant $\epsilon$ that we must add to each noise term. Then when we apply the softmax operation, we obtain the benefit that the softmax function is invariant to constant shifts in the input, i.e., $\forall \epsilon \in \mathbb{R}, \operatorname{softmax}(\boldsymbol{x}+\epsilon) = \operatorname{softmax}(\boldsymbol{x})$. Thus, we get the practical effect of softmax removing noise that we incur for not using I.I.D. Gaussian as the elements of our vectors.

\end{document}